\documentclass[journal]{IEEEtran}
\usepackage{graphicx}
\usepackage{algorithm}
\usepackage{algpseudocode}

\usepackage{xcolor}
\usepackage{amsmath,amsfonts,amsthm,amssymb,amsopn,bm,mathtools,comment}
\newtheorem{proposition}{Proposition}
\newtheorem{theorem}{Theorem}

\usepackage[compress]{cite}
\usepackage{booktabs}
\usepackage{multirow}
\usepackage{multicol}
\usepackage{subcaption}
\usepackage{float}
\usepackage{stfloats}
\usepackage[table,xcdraw]{xcolor}

\usepackage{amsmath,amsfonts,bm}

\def\1{\bm{1}}

\DeclareMathAlphabet{\mathsfit}{\encodingdefault}{\sfdefault}{m}{sl}
\SetMathAlphabet{\mathsfit}{bold}{\encodingdefault}{\sfdefault}{bx}{n}

\usepackage{colortbl}

\usepackage{xcolor,soul,framed} 

\colorlet{shadecolor}{yellow}
\usepackage{mdframed}
\usepackage{mdwmath}
\usepackage{mdwtab}
\usepackage{eqparbox}
\usepackage{url}
\usepackage{xurl}

\usepackage{hyperref} 

\usepackage{tcolorbox}

\usepackage{tabularx}

\usepackage{threeparttable}
\usepackage{array}
\usepackage[normalem]{ulem}
\usepackage{mathtools}
\usepackage{enumitem}
\usepackage{nomencl}
\usepackage{etoolbox}
\makenomenclature
\renewcommand{\nomgroup}[1]{%
  \ifthenelse{\equal{#1}{A}}{\item[\textit{\textbf{Sets}}]}{%
  \ifthenelse{\equal{#1}{B}}{\item\item[\textit{\textbf{Parameters}}]}{%
  \ifthenelse{\equal{#1}{C}}{\item\item[\textit{\textbf{Variables}}]}{}}}
}

\usepackage{amsthm}
\usepackage{tabularx}
\theoremstyle{plain}

\theoremstyle{remark}
\newtheorem{remark}{Remark}

\begin{document}
\title{
Structure-Aware Commitment Reduction for Network-Constrained Unit Commitment with Solver-Preserving Guarantees
}
\author{Guangwen~Wang,~\IEEEmembership{Student~Member,~IEEE;}
Jiaqi~Wu,~\IEEEmembership{Student~Member,~IEEE;}
Yang~Weng,~\IEEEmembership{Senior Member,~IEEE;}
and~Baosen~Zhang,~\IEEEmembership{Member,~IEEE}
\thanks{The authors Guangwen Wang, Jiaqi Wu, and Yang Weng are with the School of Electrical, Computer and Energy Engineering, Arizona State University, Tempe, AZ 
. 
Baosen Zhang is with the Department of Electrical and Computer Engineering, University of Washington, Seattle, WA.}
\vspace{-10mm}
}

\markboth{ IEEE Transactions on Power Systems}{}%

\maketitle

\begin{abstract}

The growing number of individual generating units, hybrid resources, and security constraints has significantly increased the computational burden of network-constrained unit commitment (UC), where most solution time is spent exploring branch-and-bound trees over unit-hour binary variables. To reduce this combinatorial burden, recent approaches have explored learning-based guidance to assist commitment decisions. However, directly using tools such as large language models (LLMs) to predict full commitment schedules is unreliable, as infeasible or inconsistent binary decisions can violate inter-temporal constraints and degrade economic optimality. This paper proposes a solver-compatible dimensionality reduction framework for UC that exploits structural regularities in commitment decisions. Instead of generating complete schedules, the framework identifies a sparse subset of structurally stable commitment binaries to fix prior to optimization. One implementation uses an LLM to select these variables. The LLM does not replace the optimization process but provides partial variable restriction, while all constraints and remaining decisions are handled by the original MILP solver, which continues to enforce network, ramping, reserve, and security constraints. We formally show that the masked problem defines a reduced feasible region of the original UC model, thereby preserving feasibility and enabling solver-certified optimality within the restricted space. Experiments on IEEE 57-bus, RTS 73-bus, IEEE 118-bus, and augmented large-scale cases, including security-constrained variants, demonstrate consistent reductions in branch-and-bound nodes and solution time, achieving order-of-magnitude speedups on high-complexity instances while maintaining near-optimal objective values.

\end{abstract}

\section*{Nomenclature}
\begin{description}[leftmargin=!,labelwidth=1.6cm,labelsep=0.4cm,align=left]

\item[\textit{\textbf{Sets}}]
\item[$G$] Set of generators, indexed by $g$.
\item[$N$] Set of buses, indexed by $n$.
\item[$g(n)$] Set of generators located at bus $n$.
\item[$K$] Set of transmission lines, indexed by $k$.
\item[$T$] Set of time periods, indexed by $t$.
\item[$\mathcal{F}$] Set of generator--time indices $(g,t)$ whose commitment variables are fixed prior to optimization.
\item[$\mathcal{X}$] Feasible region of the original UC formulation.
\item[$\mathcal{X}_F$] Feasible region after commitment variable restriction.

\item[\textit{\textbf{Parameters}}]
\item[$c_g$] Production cost coefficient of generator $g$.
\item[$c_g^{NL}$] No-load cost of generator $g$.
\item[$c_g^{SU}$] Startup cost of generator $g$.
\item[$P_g^{\min}$] Minimum operating capacity of generator $g$.
\item[$P_g^{\max}$] Maximum operating capacity of generator $g$.
\item[$UT_g$] Minimum up time of generator $g$.
\item[$DT_g$] Minimum down time of generator $g$.
\item[$P_k^{\max}$] Thermal limit of transmission line $k$.
\item[$d_{nt}$] Demand at bus $n$ at time $t$.
\item[$R_g^{HR}$] Ramp-rate limit of generator $g$.
\item[$R_g^{SU}$] Startup ramp limit of generator $g$.
\item[$R_g^{SD}$] Shutdown ramp limit of generator $g$.
\item[$u_{g0}$] Initial commitment state of generator $g$.
\item[$P_{g0}$] Initial dispatch level of generator $g$.
\item[$PTDF_{k,n}$] Power transfer distribution factor from bus $n$ to line $k$.

\item[\textit{\textbf{Variables}}]
\item[$P_{gt}$] Real-power output of generator $g$ at time $t$.
\item[$u_{gt}$] Commitment status of generator $g$ at time $t$.
\item[$v_{gt}$] Startup indicator of generator $g$ at time $t$.
\item[$w_{gt}$] Shutdown indicator of generator $g$ at time $t$.
\item[$f_{kt}$] Line flow on transmission line $k$ at time $t$.
\item[$\bar{u}_{gt}$] Fixed commitment value imposed on selected generator--time pairs.

\end{description}

\section{Introduction}
The unit commitment (UC) problem is a fundamental part of power system operations~\cite{padhy2004unit,knueven2020mixed}. It determines the on and off status and dispatch of generating units over a planning horizon while satisfying demand balance, reserve requirements, capacity limits, ramping constraints, and minimum up and down times \cite{hedman2010co}. If DC power flow equations are used, the UC problem becomes a mixed-integer linear program~\cite{morales2013tight}. Modern UC instances are also fairly large, often involve hundreds of generators, detailed network constraints, and multi-period horizons~\cite{knueven2020mixed,wang2025constrained}. Therefore, the resulting problems can contain tens of thousands of binary commitment variables, making it nontrivial to solve even with state-of-the-art computing resources~\cite{pandvzic2013comparison}. 

At the same time, the sizes of the models are increasing rapidly because of the need to integrate variable large loads, storage, renewable energy, and hybrid resources. The increase in complexity further enlarges the number of unit-hour binaries and strengthens temporal coupling among them \cite{Oikonomou2024WECCBaseline, Rand_QueuedUp_2024}. Consequently, the discrete search space grows rapidly with system scale, making the problem still more challenging to solve. 

To cope with the growth in problem complexity, classical acceleration techniques have focused on tightening formulations, improving warm starts, and leveraging decomposition strategies to guide or prune the MILP search \cite{chen2023security}. These approaches improve solver efficiency and reduce per-node overhead, but do not alter the number of free binary commitment variables that define the size of the discrete search space \cite{ morales2013tight}. As system models become larger and more granular, the branch-and-bound process remains the main computational bottleneck.

More recently, learning-based approaches have been proposed to guide commitment decisions and accelerate the MILP solution process. Some methods predict unit-hour commitment patterns from forecasts and covariates, using supervised models to generate warm starts or partially fixed commitment variables \cite{pineda2022learning, zhang2025learning}. Others integrate learned signals into the MILP solver, such as through branching guidance or selective binary assignments that aim to reduce branch-and-bound exploration \cite{zhuozan2024combining,lawless2025llmsjournal}. In addition, learning has been combined with decomposition-based UC formulations to estimate subproblem quality and prioritize promising substructures. 

Although these techniques lead to speedups in stable operating regimes, they rely heavily on historical labeled data and often attempt to approximate or replace part of the discrete search \cite{11225172}. But generating a good set of training data can be prohibitively expensive, given the need to solve a large number of MILP problems. More importantly, when system conditions shift or new resources are introduced, learned commitment patterns can lose accuracy and produce infeasible or inconsistent decisions \cite{wu2025unified}. The solver must then override or correct this guidance, often requiring additional MILP solves that reduce the net runtime gains. Structurally, these approaches do not reduce the number of free commitment binaries in a way that preserves solver-certified feasibility and optimality. Instead, they rely on prediction in a problem where exact constraint enforcement and discrete consistency are essential. As we can see, these approaches do not explicitly exploit structural regularities in commitment decisions to reduce the dimensionality of the binary search space.

To address this structural limitation, we propose a structure-aware commitment reduction framework that exploits the observation that many unit-hour commitment decisions remain stable across similar operating conditions. Rather than relying on prediction to generate full schedules, the method identifies a sparse subset of structurally stable commitment variables that can be fixed prior to optimization. One practical implementation of this idea uses a large language model (LLM) to select these variables. 
This design reduces the number of free commitment binaries while leaving the original UC formulation unchanged. The MILP solver then determines the remaining commitment and dispatch decisions and enforces all operational and network constraints exactly. By reducing the number of free commitment binaries while retaining solver authority, the method directly reduces branch-and-bound exploration without compromising feasibility. We show that even by fixing a small number of commitment variables, we can drastically reduce the solving time. We formally prove that the modified problem defines a subset of the original UC feasible region, ensuring that feasibility is preserved by the solver.

This restriction-based framework operates entirely within the original UC formulation and does not alter the underlying optimization model. In addition to the commitment fixing strategy, we introduce a pre-solve feasibility check that verifies necessary logical consistency conditions before invoking the MILP solver. This step avoids launching solves that are trivially infeasible and reduces redundant solver calls under strict market timelines. The number of fixed commitment binaries is explicitly limited to avoid excessive restriction of the feasible region and to control objective deviation. We formally show that the resulting problem defines a subset of the original UC feasible region. Therefore, feasibility is preserved by the solver and optimality is certified within the reduced commitment space. Together, these design choices reduce branch-and-bound exploration while preserving exact constraint enforcement, and they remain effective even when historical commitment data are limited or system configurations change.

The proposed framework is evaluated on the IEEE 57-bus, RTS 73-bus, and IEEE 118-bus systems, along with augmented 118-bus cases with additional generators, to examine scalability under increasing commitment dimensionality. Experiments include a setting without access to historical commitment schedules, load uncertainty perturbations, and a year-long sequence of daily UC instances. Baselines include the original MILP solve and data-driven commitment fixing strategies based on clustering and nearest-neighbor matching, Performance is evaluated using solve time, branch-and-bound nodes, simplex iterations, and objective deviation relative to the MILP reference, with feasibility preserved by the unchanged UC formulation. Results show consistent reductions in explored nodes and wall-clock time across system sizes, achieving substantial speedups while maintaining near-optimal objective values. Results show consistent reductions in branch-and-bound nodes and solve time across system sizes. The proposed method achieves up to $6.7$x speedup (e.g., on the RTS $73$-bus system) with less than $0.3\%$ objective deviation from the MILP baseline.
These improvements remain even when no historical commitment schedules are provided. This demonstrating effectiveness under limited training data and changing system conditions. 

We note that there is an orthogonal line of work that uses LLM to translate natural language description to optimization problems~\cite{ahmed2024lm4opt,cai2025mingyue,zhang2025poweragent,muhao2025solar,jimaging12030106}. In contrast, our paper is on solving the underlying optimization problem using LLMs and our results can be integrated into these or more conventional solvers. 

The rest of this paper is organized as follows. Section \ref{sec:problem_formulation} presents the problem formulation for classic UC and LLM-Assisted UC problem formulation. Section \ref{sec:restricted_uc} details task decomposition, context engineering and prompt design. Section \ref{sec:llm-theory} proves some analytical insights into our proposed framework. Section \ref{sec:experiments} provides numerical validation of the proposed approaches. Conclusions and discussions are presented in Section \ref{sec:conclusion}.

\section{Problem Formulation}
\label{sec:problem_formulation}

We consider a network-constrained unit commitment (UC) problem over a discrete time horizon $T$. Let $G$ denote the set of generators and $N$ the set of buses. The objective is to determine the binary commitment decisions and continuous dispatch levels that minimize total production cost while satisfying operational and transmission constraints under a DC power flow approximation. Let $u_{gt} \in \{0,1\}$ denote the on/off commitment status of generator $g$ at time $t$, $P_{gt}$ its real-power output, and $v_{gt}$ and $w_{gt}$ the startup and shutdown indicators, respectively. The network-constrained UC formulation is given by

\begin{subequations}\label{eq:uc}
\begin{align}
\min \quad
& \sum_{t \in T} \sum_{g \in G}
\left( c_g P_{gt} + c_g^{NL} u_{gt} + c_g^{SU} v_{gt} \right)
\label{eq:uc_obj}
\\
\text{s.t.} \quad
& P_g^{\min} u_{gt} \le P_{gt} \le P_g^{\max} u_{gt},
&& \forall g,t 
\label{eq:uc_cap}
\\
& u_{gt} - u_{g,t-1} = v_{gt} - w_{gt},
&& \forall g,t 
\label{eq:uc_logic}
\\
& \sum_{s=t-UT_g+1}^{t} v_{gs} \le u_{gt},
&&\hspace{-6.3em} \forall g,\ t = UT_g,\ldots,T
\\
& \sum_{s=t-DT_g+1}^{t} w_{gs} \le 1 - u_{gt},
&& \notag \\
& &&\hspace{-6.3em} \forall g,\ t = DT_g,\ldots,T
\label{eq:uc_DT}
\\
& P_{gt} - P_{g,t-1}
\le R_g^{HR} u_{g,t-1} + R_g^{SU} v_{gt},
&& \forall g,t
\label{eq:uc_ramp_up}
\\
& P_{g,t-1} - P_{gt}
\le R_g^{HR} u_{gt} + R_g^{SD} w_{gt},
&& \forall g,t
\label{eq:uc_ramp_down}
\\
& \sum_{g \in g(n)} P_{gt} - d_{nt}
= \sum_{k \in K} a_{n,k} f_{kt},
&& \forall n,t
\label{eq:uc_balance}
\\
& f_{kt} = b_k (\theta_{i_k,t} - \theta_{j_k,t}),
&& \forall k,t
\label{eq:uc_flow}
\\
& -P_k^{\max} \le f_{kt} \le P_k^{\max},
&& \forall k,t
\label{eq:uc_line_limit}
\\
& u_{gt}, v_{gt}, w_{gt} \in \{0,1\},
&& \forall g,t.
\label{eq:uc_domain}
\end{align}
\end{subequations}

In \eqref{eq:uc}, the objective \eqref{eq:uc_obj} minimizes total generation, no-load, and startup costs. Constraints \eqref{eq:uc_cap}–\eqref{eq:uc_ramp_down} enforce generator capacity, logical consistency, minimum up/down time, and ramping limits. Constraints \eqref{eq:uc_balance}–\eqref{eq:uc_line_limit} enforce nodal power balance and transmission security under a DC power flow model. One bus is selected as the reference angle.

\medskip
\noindent
\textbf{Problem Statement:} Given system parameters and demand profiles, solve the network-constrained UC problem \eqref{eq:uc} to obtain a globally optimal commitment and dispatch schedule under the mixed-integer formulation.

\section{Commitment-Restricted Unit Commitment Framework}
\label{sec:restricted_uc}

Solving \eqref{eq:uc} directly as a single MILP can be computationally demanding because branch-and-bound exploration over unit-hour commitment variables dominates runtime in large systems. To reduce the discrete search dimension while preserving the original UC structure, we introduce a commitment variable restriction framework.

\subsection{Restricted Unit Commitment Formulation}

A large language model (LLM) selects a subset of generator–time indices
$
\mathcal{F} \subseteq G \times T
$
and corresponding fixed values $\bar{u}_{gt} \in \{0,1\}$ for $(g,t) \in \mathcal{F}$. 
The restriction is imposed directly in the UC formulation through
\begin{align}
u_{gt} = \bar{u}_{gt},
\quad \forall (g,t) \in \mathcal{F}.
\label{eq:commitment_fixing}
\end{align}

All remaining decision variables remain free and are optimized by the MILP solver subject to the full set of constraints in \eqref{eq:uc}. 
Before incorporating the selected restriction set $\mathcal{F}$, a pre-solve feasibility verification step checks necessary logical consistency conditions, including ramping feasibility and basic supply–demand consistency. The goal is to prevent restrictions that would make the UC problem infeasible. 
Let $\mathcal{X}$ denote the feasible region of the original UC formulation and $\mathcal{X}_F$ the feasible region after imposing \eqref{eq:commitment_fixing}. By construction, $\mathcal{X}_F \subseteq \mathcal{X}$. Therefore, any solution returned by the solver remains feasible for the original UC model, and global optimality is certified within the reduced feasible region $\mathcal{X}_F$.

\begin{figure}[h!]
  \centering
  \includegraphics[width=\columnwidth]{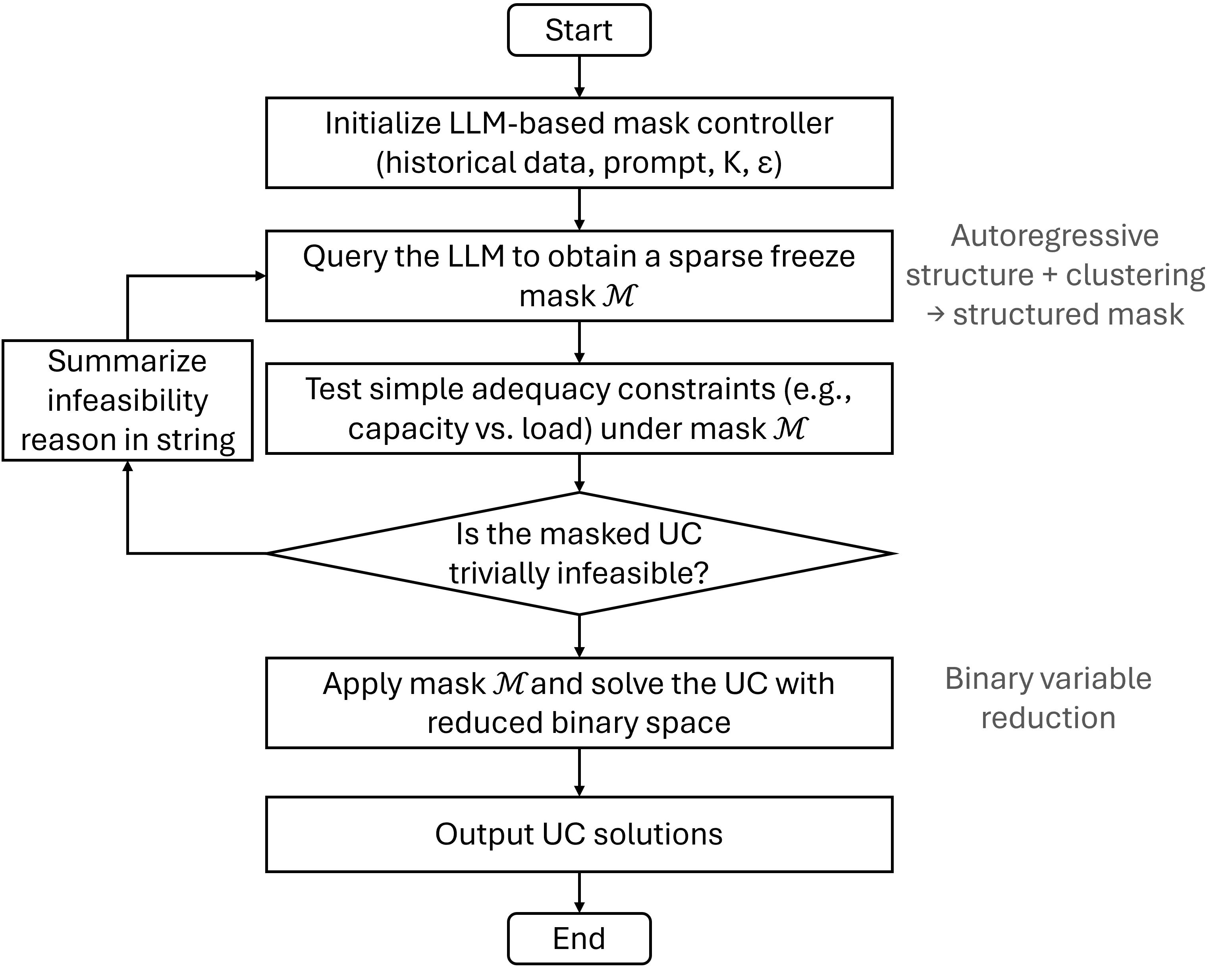}
  \caption{Workflow of the proposed solver-preserving commitment variable reduction method. An LLM generates a sparse commitment restriction set, a feasibility screening step removes inconsistent selections, and the resulting reduced UC problem is solved by a MILP solver with full constraint enforcement and certified optimality within the restricted space.}
  \label{fig:llm-uc-flow}
  \vspace{-4mm}
\end{figure}

\subsection{LLM-Based Restriction Generation}

Our solver-coupled design employs an LLM to generate a sparse freeze mask that fixes only a small fraction of binary on/off decisions, thereby reducing the combinatorial search space for MILP solvers. By coupling high-level guidance from the LLM with exact optimization, the method preserves feasibility and solution quality while accelerating solution time. 
When historical UC solutions are available, we select the $H$ most load-similar past days based on the 24-hour demand profile. 
For each selected day, we provide its 24-hour demand vector and the corresponding MILP commitment schedule as a compact reference. 
This guides the LLM to identify commitment decisions that remain consistent across similar operating conditions, rather than generating a full schedule.
In addition, generator characteristics (e.g., $P_{\min}$/$P_{\max}$, minimum up/down times, and initial on/off status) are provided once as a JSON object, while the target-day 24-hour nodal demands are supplied as a simple array. 
No network, ramping, or reserve parameters are exposed to the model, and these elements remain strictly within the solver.

We decompose the mask-generation prompt into distinct parts, including task description, input context, output specification, and mask rules, to ensure clarity and consistency. The model is explicitly instructed not to produce a full schedule, but to return a strict JSON-formatted array of tuples $(t,g,u)$ indicating that unit~$g$ should be fixed to state $u\in\{0,1\}$ at hour~$t$. A per-hour limit $K$ is imposed on the number of commitment variables that may be fixed, preventing excessive restriction of the feasible region. The LLM is configured to produce consistent outputs across repeated runs. When uncertainty arises, the model is instructed to restrict fewer units and defer the remaining commitment decisions to the MILP solver. The prompt below illustrates this structure.

\begin{tcolorbox}[colback=gray!5!white,colframe=gray!75!black,title=\footnotesize Commitment Restriction Specification Provided to the LLM]
\footnotesize\ttfamily
\textbf{Task:} Identify a sparse set of unit--hour commitment variables to fix in a 24-hour UC problem. 
Do not generate a complete schedule. Only specify selected commitment decisions; 
all remaining decisions are determined by the MILP solver.

\

\textbf{Input Data:}
\begin{itemize}
  \item Generator parameters for each unit $g$: 
        $P_g^{\min}$, $P_g^{\max}$, minimum up/down times, and initial status.
  \item Hourly demand profile for the 24-hour horizon.
\end{itemize}

\

\textbf{Output Format:}
Return a structured list of tuples $(t,g,u)$ indicating that 
unit $g$ is fixed to commitment state $u \in \{0,1\}$ at hour $t$. 
No explanatory text should be included.

\

\textbf{Restriction Guidelines:}
\begin{enumerate}
  \item At each hour $t$, fix at most $K$ units to avoid excessive feasible-region reduction.
  \item Favor commitment patterns consistent with minimum up/down requirements and initial conditions.
  \item When uncertain, restrict fewer units and allow the solver to determine the remaining commitments.
\end{enumerate}
\end{tcolorbox}

The structured specification above ensures that the model output is well-defined, reproducible, and directly compatible with existing UC implementations. By separating task description, input data, and restriction guidelines, the interface promotes consistent decision patterns across operating scenarios.
For each daily UC instance, the model receives the generator parameters and the 24-hour demand profile, and returns a structured list of unit--hour commitment restrictions. Because the output follows a predefined schema, the selected $(t,g,u)$ tuples can be incorporated directly into the UC model without additional interpretation, enabling seamless integration with standard MILP-based workflows.

\subsection{Applying Restrictions and Solving the MILP}

Once a freeze mask $\mathcal{M}=[(t,g,u)]$ is returned, it is injected into the UC model. For each tuple $(t,g,u)$ in the mask, the corresponding variable is set to $u$ and fixed. All remaining operational and network constraints are enforced by the MILP solver, subject to the fixed commitment decisions.

\begin{tcolorbox}[colback=gray!5!white,colframe=gray!75!black,
title=\footnotesize Reduced UC Solution Procedure]
\footnotesize\ttfamily
\textbf{1) Data initialization:} Load network topology, transmission limits, 
generator parameters, and the time-indexed demand profile. \\

\textbf{2) Scenario update:} Set the hourly demand values for the current UC instance. \\

\textbf{3) Commitment reset:} Release previously fixed commitment variables to ensure 
a clean model state. \\

\textbf{4) Commitment restriction:} For each $(t,g,u)\in\mathcal{M}$, 
fix the corresponding unit--hour commitment variable in the UC formulation. \\

\textbf{5) MILP optimization:} Solve the resulting reduced UC MILP with full 
enforcement of operational, inter-temporal, and transmission constraints. \\

\textbf{6) Result extraction:} Obtain the commitment schedule, dispatch levels, 
and objective value for the restricted problem.
\end{tcolorbox}

\subsection{Solution Procedure with Feedback}

As summarized in Algorithm~\ref{alg:uc-feedback}, the procedure takes as inputs the generator parameters, hourly demand profile, restriction limit $K$, and optional screening thresholds, and outputs the final commitment and dispatch schedules. An optional pre-solve verification step is used to detect restriction sets that would render the UC problem infeasible. A post-solve validation step may be applied to confirm that operational and 
network constraints remain satisfied under the restricted solution. Together, these steps provide a structured integration of commitment restriction within a standard MILP-based UC solution process.

\begin{algorithm}[h!]
\caption{LLM-Assisted UC with masking and feedback}
\label{alg:uc-feedback}
\begin{algorithmic}[1]
  \Require Generator metadata $\mathcal{G}$, nodal demands $\{P_d(t)\}_{t=1}^{24}$, freeze cap $K$, optional screening thresholds
  \Ensure Commitment schedule $\{U_{t,g}\}$ and dispatch $\{P_{g, t}\}$
  \State Build the UC MILP model with network and operational constraints
  \State Override bus loads to match the current scenario
  \State Query the LLM with the structured prompt to obtain a sparse mask $\mathcal{M}$
  \If{pre-solver screening is enabled}
    \State Check whether the unfrozen generators can meet demand at each hour; if not, request a revised mask
  \EndIf
  \State \textbf{unfix} all commitments $U_{t,g}$; \textbf{fix} $U_{t,g}\leftarrow u$ for each $(t,g,u)\in\mathcal{M}$
  \State \textbf{solve} the UC MILP; record the resulting commitments and dispatch
  \If{post-solver validation is enabled}
    \State Fix $U_{t,g}$ and run a fast economic dispatch or DC/AC flow to confirm reserve and line limits
    \If{violations occur}
      \State Summarise the binding constraints and request a revised mask
    \EndIf
  \EndIf
  \State \Return $\{U_{t,g}\}$, $\{P_{g, t}\}$, objective value, and any validation logs
\end{algorithmic}
\end{algorithm}
\vspace{-2mm}

\section{Structural Basis for Commitment Restriction in Unit Commitment}
\label{sec:llm-theory}

This section provides a structural explanation for why data-driven commitment restriction can be effective in network-constrained unit commitment (UC) problems. Rather than viewing the LLM as a predictive optimizer, we interpret its role as identifying commitment variables that are structurally stable under recurring operating regimes. The solver remains responsible for enforcing all operational and network constraints, while the restriction mechanism seeks to reduce the dimensionality of the discrete search space by fixing commitment decisions that are unlikely to change across feasible and near-optimal solutions.

In practical power system operations, commitment schedules are not arbitrary. They are strongly influenced by load trajectories, generator characteristics, minimum up/down requirements, and ramping capabilities. Across days with similar demand patterns and resource availability, many commitment decisions exhibit regular and repeatable structure. For example, large base-load units are typically committed during sustained high-demand periods, while peaking units are activated during predictable load ramps. Similarly, units with long minimum up times, once committed, are typically required to remain online for multiple consecutive hours, resulting in contiguous commitment blocks in the schedule.
The proposed restriction framework leverages this structural regularity. Instead of predicting the full commitment schedule, the LLM identifies a sparse subset of unit–hour commitment variables whose values are frequently stable under similar operating conditions. These variables are then fixed prior to optimization, thereby reducing the number of binary decisions explored in branch-and-bound, while leaving the remainder of the UC problem unchanged.

\subsection{Operating Regimes and Commitment Regularity}

Daily load trajectories define recurring operating regimes, including low-load overnight hours, morning ramping periods, daytime high-demand intervals, and evening peak conditions. Although the exact dispatch levels and marginal costs may vary across days, the qualitative structure of commitment decisions often remains similar under similar operating conditions. 
Let $x$ denote the input consisting of generator metadata and a 24-hour load profile. The restriction mechanism can be interpreted as mapping $x$ to a sparse commitment restriction set $\mathcal{M}(x)$. When two input scenarios $x_1$ and $x_2$ correspond to similar operating conditions, such as comparable peak demand levels and ramping characteristics, their resulting commitment patterns often share common structural elements. In particular, certain commitment variables may consistently take the same values across feasible and near-optimal solutions under those conditions. This motivates the approximate relation
\begin{align}
    \mathcal{M}(x_1) \approx \mathcal{M}(x_2),
\end{align}
when $x_1$ and $x_2$ represent similar operating regimes.

Notably, this mechanism does not imply that the full commitment schedule is uniquely determined by the demand profile. In network-constrained UC, multiple feasible commitment combinations may satisfy the same load conditions, particularly for marginal and fast-start units. These units often provide flexibility and can switch status across near-optimal solutions. In contrast, commitment decisions for large base-load generators or for hours with tight capacity margins tend to be more constrained and therefore exhibit limited variability across feasible solutions. By fixing commitment variables that consistently take the same value under similar loading conditions, the proposed framework reduces the effective number of binary variables explored in branch-and-bound, while leaving economically sensitive decisions to the MILP solver.

\subsection{Temporal Coupling in Unit Commitment}

Unit commitment decisions are inherently inter-temporal due to minimum up-time, minimum down-time, and ramping constraints. The commitment variable $u_{g,t}$ is not independent across hours. Once a unit is started, minimum up-time constraints require it to remain online for a prescribed number of subsequent periods. Similarly, shutdown decisions impose minimum down-time restrictions, and ramp-rate limits constrain feasible transitions between consecutive dispatch levels. 
As a result, feasible commitment schedules exhibit contiguous commitment intervals rather than isolated hour-by-hour switching. For generators with long minimum up-time requirements or limited ramping flexibility, commitment decisions over several consecutive hours are strongly coupled. In many practical instances, once such a unit is committed during a high-demand interval, its status over the next several hours is effectively constrained by operational limits.

This temporal coupling provides a structural opportunity for binary reduction. If a generator is consistently committed over a sustained interval under similar loading conditions, fixing the corresponding sequence of unit-hour variables can reduce the depth and width of branch-and-bound exploration. Notably, the proposed approach restricts only a subset of commitment variables, leaving flexible and economically marginal decisions to the MILP solver. 
To avoid infeasible reductions, a pre-solve screening step verifies basic consistency conditions before fixing commitments. These checks include aggregate capacity sufficiency and ramping feasibility under the proposed fixed decisions. The full UC formulation, including network and operational constraints, is then enforced by the MILP solver without relaxation. The number of restricted variables per hour is explicitly capped to limit feasible-region reduction and to control potential objective deviation.

\subsection{Feasibility and Optimality Properties}
\label{subsec:feasibility_properties}

We now formalize the structural properties of the proposed restriction mechanism relative to the original network-constrained UC formulation. Let $\mathcal{F}$ denote the feasible region of the original UC problem defined by constraints (1b)–(1j). Let $M \subseteq \{(g,t)\}$ denote the set of restricted unit-hour pairs, and define the restricted feasible region
\[
\mathcal{F}_M := \{ (u,P,v,w,f,\theta) \in \mathcal{F} \;|\; u_{g,t} = \bar{u}_{g,t}, \ \forall (g,t) \in M \},
\]
where $\bar{u}_{g,t} \in \{0,1\}$ are the fixed commitment values.

\begin{proposition}[Feasibility Preservation]
If $\mathcal{F}_M$ is nonempty, then any solution returned by the MILP solver for the restricted problem satisfies all constraints of the original UC formulation.
\end{proposition}

\begin{proof}
The restricted UC problem is obtained by adding equality constraints $u_{g,t}=\bar{u}_{g,t}$ for $(g,t)\in M$ to the original MILP. All other constraints remain unchanged. Therefore, any feasible solution of the restricted problem satisfies the full set of original UC constraints by construction.
\end{proof}

\begin{theorem}[Reduced Feasible Region]
The restricted feasible region satisfies
\(
\mathcal{F}_M \subseteq \mathcal{F}.
\)
Consequently,
\[
\min_{x \in \mathcal{F}} f(x) 
\;\le\;
\min_{x \in \mathcal{F}_M} f(x),
\]
where $f(x)$ denotes the UC objective.
\end{theorem}

\begin{proof}
By definition, $\mathcal{F}_M$ is formed by intersecting $\mathcal{F}$ with additional linear equality constraints. Hence $\mathcal{F}_M \subseteq \mathcal{F}$. The objective inequality follows directly from minimization over a subset.
\end{proof}

\begin{remark}[Optimality in Restricted Space]
If $x_M^\star$ is the optimal solution of the restricted UC problem, then $x_M^\star$ is globally optimal over $\mathcal{F}_M$. The MILP solver therefore certifies optimality within the reduced commitment space, although global optimality over $\mathcal{F}$ is not guaranteed in general. However, as we will show in Sec.~\ref{sec:experiments}, the optimality loss is minuscule, while the solving efficiency gain is significant. 
\end{remark}

These results clarify that the proposed framework preserves the original UC constraints and exact optimization within the restricted subspace. Any deviation in objective value arises solely from feasible-region reduction through selective commitment restriction, rather than from approximation of the underlying optimization model.



\section{Experimental Validation} 
\label{sec:experiments}

We evaluates the proposed commitment-restriction framework for accelerating network-constrained unit commitment (UC). 
The method fixes a carefully selected subset of unit–hour commitment variables prior to solving, thereby reducing the effective branch-and-bound search space. 
The underlying UC formulation, including all network constraints, ramping limits, and minimum up/down requirements, remains unchanged. 
The MILP solver therefore continues to enforce full operational feasibility and determines the optimal dispatch together with the remaining unfixed commitment decisions.

We compare the proposed strategy with representative data-driven acceleration approaches that derive commitment patterns from historical UC instances. 
Such methods typically reduce computational effort by transferring commitment structures from similar load conditions, but they may not explicitly guarantee feasibility during prediction. 
We also compare against approaches that preserve the original UC formulation while providing structured guidance to the solver, so that feasibility and optimality certification remain entirely within the MILP framework.

The numerical studies are organized to address several practical questions that arise in UC acceleration. 
First, we verify that computational improvements do not result from arbitrarily fixing commitment variables, but from selecting decisions that are consistent with system operating conditions. 
This is demonstrated by comparing the proposed restriction strategy with a random commitment-fixing benchmark under the same fixing ratio.

Second, we evaluate performance in a setting where no historical commitment schedules are available to guide the restriction process. 
This experiment examines whether the method avoids economically or operationally detrimental fixing decisions when only generator characteristics and the daily load profile are provided.

Third, we examine performance under varying demand conditions. 
In practice, commitment patterns derived from historical data may become less representative when overall demand levels or ramping requirements change. 
We therefore test whether the proposed method maintains computational effectiveness and acceptable cost performance when applied to load profiles that differ from those seen in prior instances.

Finally, we evaluate robustness under load uncertainty and over extended operating horizons. 
Specifically, we examine whether the computational benefits persist when load forecasts are perturbed and when the UC problem is solved repeatedly over long sequences of daily operations. 
Such settings introduce stronger inter-temporal interactions and increased overall computational burden, reflecting practical studies such as rolling-horizon scheduling and scenario-based planning.

\subsection{Experimental Setup}
\label{sec:setup}

The proposed approach is evaluated on UC systems across multiple scale levels, including the IEEE 57-bus system, the IEEE 73-bus RTS test system \cite{barrows2020reliability}, and the IEEE 118-bus system. For the IEEE 57-bus and IEEE 118-bus cases, the network topology (bus/branch data and electrical limits) is obtained from PGLib-OPF/MATPOWER representations \cite{pglib_uc_2019_08}. The corresponding unit commitment parameters (e.g., minimum up/down times, ramp limits, start-up/shut-down costs, and multi-segment production costs) are taken from the MATPOWER-format UC instances provided by \cite{unitcommitmentjl, matpower2011}. These sources are combined to construct network-constrained UC MILPs.

To study how UC difficulty escalates with increased commitment complexity, we construct two augmented variants on the IEEE 118-bus network by adding additional generating units, resulting in two enlarged MILP cases denoted as 118-bus+20 and 118-bus+30. These augmented systems increase the number of binary commitment variables and inter-temporal coupling constraints.
All test cases are implemented in AMPL using the identical UC formulation and solved with Gurobi under consistent parameter settings. Solver tolerances and numerical configurations are kept unchanged across all systems and all solution strategies to ensure fair comparisons.

For robustness evaluation, the IEEE 73-bus RTS system is further subjected to load perturbations. The baseline hourly demand profile is perturbed using Gaussian noise with relative magnitudes between $5\%$ and $30\%$, applied independently to each hour to generate new daily instances. Each perturbed case uses the same generator data and network configuration as the baseline. 
For the long-horizon study, daily demand profiles are derived from real historical data provided by the PGLIB benchmark. Noise is applied where necessary to construct a complete 365-day sequence of UC instances.
The solving times are recorded in all comparisons corresponding to the internal solver time returned by Gurobi.


Four solution strategies are evaluated on each system using the identical UC formulation:

\begin{itemize}
    \item \textbf{Original MILP:} The full UC formulation solved directly by Gurobi without any restriction. This serves as the reference for objective values and computational effort.
    
    \item \textbf{K-means Masking:} Historical days are clustered using K-means based on load profiles. Representative commitment patterns derived from cluster centroids are used to fix selected unit-hour variables.
    
    \item \textbf{KNN Masking:} Nearest-neighbor retrieval is performed using load similarity. Commitment patterns from the most similar historical days are transferred to define restriction sets.
    
    \item \textbf{LLM-Assisted Commitment Restriction (GPT-4o):} A large language model generates a sparse subset of unit-hour commitment decisions to be fixed prior to solving. The UC formulation itself remains unchanged, and feasibility and optimality certification are fully enforced by the MILP solver.
\end{itemize}

The data-driven masking approaches rely on historical UC instances to extract representative commitment structures. In contrast, the LLM-Assisted strategy uses generator characteristics and load information to identify commitment variables that are likely to remain stable over the scheduling horizon. In all cases, once commitment variables are fixed, the remaining UC problem is solved using the same MILP solver and formulation.

\subsection{Comparison with Random Commitment Restriction}

This subsection verifies that restricting an appropriate subset of commitment variables can significantly reduce the computational burden of unit commitment, whereas restricting an inappropriate subset can increase both the objective value and the computational burden. To make this distinction explicit, we contrast a random freezing strategy with the proposed LLM-Assisted masking approach on a representative 24-hour instance of the IEEE 73-bus RTS system, using the same freeze ratio 10\% for both methods. Table~\ref{tab:milp_random_gpt} reports the resulting simplex iterations, branch-and-bound nodes, and solving times, together with the baseline MILP solve without any freezing. 

\begin{table}[h]
\centering
\caption{Comparison among plain MILP, random freezing, and LLM-Assisted.}
\label{tab:milp_random_gpt}
\setlength{\tabcolsep}{4pt}
\begin{tabular*}{\linewidth}{@{\extracolsep{\fill}}lccc}
\toprule
Method & Simplex iters & Branch nodes & Time (s) \\
\midrule
MILP (no freezing)      & $3.61\times10^{6}$ & $3.28\times10^{5}$ & $1149.33$ \\
Random freezing ($\approx10\%$) & $1.49\times10^{7}$ & $1.19\times10^{6}$ & $3004.75$  \\
LLM-Assisted ($\approx10\%$) & ${\bm{7.07\times10^{5}}}$  & ${\bm{6.26\times10^{4}}}$  & ${\bm{177.60}}$   \\
\bottomrule   
\end{tabular*}
\end{table}

Under the random commitment-fixing strategy, the solver requires substantially more computational effort than the baseline MILP. This is reflected in higher simplex iterations, a larger number of branch-and-bound nodes, and increased wall-clock time. In addition, random fixing degrades solution quality. The total dispatch cost increases by approximately 6.1\% relative to the baseline MILP solution. This cost increase indicates that indiscriminate commitment fixing can restrict the feasible region in a way that interferes with the solver’s branch-and-bound search, even when the fixing ratio is small. 

In contrast, the proposed LLM-Assisted commitment restriction maintains solution quality while reducing computational effort. The objective deviation remains below $0.3\%$ relative to the baseline MILP solution. At the same time, simplex iterations and branch-and-bound nodes are reduced, leading to shorter solve times. These results indicate that computational improvement depends on selecting structurally appropriate commitment variables for restriction, rather than on fixing variables arbitrarily.

\subsection{Performance Without Historical Commitment Information}

The preceding results indicate that effective UC acceleration depends on restricting commitment variables in a manner consistent with system operating conditions. We next evaluate the proposed method in a setting where no historical commitment schedules are provided to guide the restriction process. 
In this experiment, the commitment-restriction strategy is generated using only generator characteristics and the test-day 24-hour load profile, which is perturbed from real system data by $10\%$. No prior commitment solutions or historical UC instances are used. 
Performance is compared against the baseline UC MILP solved directly by Gurobi. Data-driven masking methods such as KNN and K-means are not included in this experiment because they require historical commitment schedules to construct restriction patterns.

\begin{table}[h!]
\centering
\caption{Performance comparison of the baseline MILP and the LLM-Assisted commitment restriction without historical commitment information. Average objective value and solve time are reported over test instances.}
\label{tab:uc_zero_shot_scale}
\small
\begin{tabularx}{\columnwidth}{l l >{\raggedleft\arraybackslash}X >{\raggedleft\arraybackslash}X}
\toprule
System & Method & Avg Cost & Avg Time(s) \\
\midrule

\multirow{2}{*}{57-bus}
 & MILP & $1.12\times 10^{6}$ & $5.85$ \\
 & LLM  & $1.12\times 10^{6}$ & $4.74$ \\
\midrule

\multirow{2}{*}{RTS}
 & MILP    & $3.30\times 10^{6}$ & $1149.33$ \\
 & LLM     & $3.37\times 10^{6}$ & $189.12$ \\
\midrule

\multirow{2}{*}{118-bus}
 & MILP & $1.67\times 10^{7}$ & $9.06$ \\
 & LLM  & $1.70\times 10^{7}$ & $6.35$ \\
\midrule

\multirow{2}{*}{118-bus+20}
 & MILP & $1.34\times 10^{7}$ & $45.99$ \\
 & LLM  & $1.35\times 10^{7}$ & $22.26$ \\
\midrule

\multirow{2}{*}{118-bus+30}
 & MILP & $1.33\times 10^{7}$ & $823.25$ \\
 & LLM  & $1.35\times 10^{7}$ & $155.34$ \\
\bottomrule
\end{tabularx}
\end{table}

Table~\ref{tab:uc_zero_shot_scale} summarizes the performance of the proposed commitment-restriction strategy when no historical commitment schedules are available. The method reduces solve time across all tested systems, with the largest gains observed on the more computationally challenging instances. 

For the 57-bus system, solve time decreases from $5.85\,\mathrm{s}$ to $4.74\,\mathrm{s}$, corresponding to a $1.23\times$ speedup with a $0.27\%$ objective increase. On the RTS system, the baseline MILP requires $1149.33\,\mathrm{s}$ on average, while the proposed method reduces this to $189.12\,\mathrm{s}$, yielding a $6.08\times$ speedup with a $2.18\%$ objective increase. For the 118-bus family, speedups of $1.43\times$, $2.07\times$, and $5.30\times$ are observed on 118-bus, 118-bus+20, and 118-bus+30, respectively. As additional generators are introduced, the MILP solve time increases from $9.06\,\mathrm{s}$ to $45.99\,\mathrm{s}$ and further to $823.25\,\mathrm{s}$, whereas the corresponding solve times under the proposed restriction grow more moderately to $6.35\,\mathrm{s}$, $22.26\,\mathrm{s}$, and $155.34\,\mathrm{s}$.

In terms of solution quality, the restricted solutions remain close to the MILP baseline, with objective deviations ranging from $0.27\%$ to $2.18\%$. The larger deviation observed on the RTS system suggests that, in the absence of historical commitment information, restricting certain commitment variables may exclude economically preferable schedules, particularly in systems with strong inter-temporal coupling and complex commitment structure. Overall, these results indicate that meaningful computational acceleration can be achieved without historical commitment data, although objective deviations may increase on more challenging systems.

\subsection{Comparison Across System Levels and Generator Augmentation}

In addition to the comparison with random commitment fixing and the case without historical commitment information, we examine how the proposed acceleration strategy scales with system size. Specifically, performance is evaluated across multiple system levels (57-bus, RTS, and 118-bus). 
To further increase problem difficulty, we augment the 118-bus system by adding additional generating units, resulting in two expanded cases denoted as 118-bus+20 and 118-bus+30. These augmented systems increase the number of binary commitment variables and strengthen inter-temporal coupling constraints. Such increases are known to significantly enlarge the branch-and-bound search space and computational burden of solving the UC MILP.

For each system, we compare the full MILP formulation with three acceleration strategies: the proposed LLM-Assisted commitment restriction, KNN-based masking, and K-means-based masking. Test instances are constructed from real system data with a $10\%$ load perturbation to assess robustness. Table~\ref{tab:uc_avg_cost_time} reports the average objective value and the average solve time for each method.

\begin{table}[h!]
\centering
\caption{Average objective value and solve time of the compared UC solution strategies across different system sizes.}
\label{tab:uc_avg_cost_time}
\small
\begin{tabularx}{\columnwidth}{l l >{\raggedleft\arraybackslash}X >{\raggedleft\arraybackslash}X}
\toprule
System & Method & Avg Cost & Avg Time(s) \\
\midrule

\multirow{4}{*}{57-bus}
 & MILP & $1.12 \times 10^{6}$ & $5.85$ \\
 & LLM  & \bm{$1.12 \times 10^{6}$} & \bm{$3.44$} \\
 & KNN  & $1.12 \times 10^{6}$ & $3.91$ \\
 & K-means   & $1.12 \times 10^{6}$ & $3.84$ \\
\midrule

\multirow{4}{*}{RTS}
 & MILP    & $3.30 \times 10^{6}$ & $1149.33$ \\
 & LLM     & \bm{$3.30 \times 10^{6}$} & \bm{$171.59$} \\
 & KNN     & $3.31 \times 10^{6}$ & $558.34$ \\
 & K-means & $3.31 \times 10^{6}$ & $714.54$ \\
\midrule

\multirow{4}{*}{118-bus}
 & MILP & $1.68 \times 10^{7}$ & $9.06$ \\
 & LLM  & \bm{$1.69 \times 10^{7}$} & \bm{$4.97$} \\
 & KNN  & $1.69 \times 10^{7}$ & $5.72$ \\
 & K-means   & $1.69\times 10^{7}$ & $6.92$ \\
\midrule

\multirow{4}{*}{118-bus+20}
 & MILP & $1.34\times 10^{7}$ & $45.99$ \\
 & LLM  & \bm{$1.34\times 10^{7}$} & \bm{$23.49$} \\
 & KNN  & $1.34\times 10^{7}$ & $30.69$ \\
 & K-means   & $1.34\times 10^{7}$    & $44.17$ \\
\midrule

\multirow{4}{*}{118-bus+30}
 & MILP & $1.333\times 10^{7}$ & $823.25$ \\
 & LLM  & \bm{$1.335\times 10^{7}$} & \bm{$112.49$} \\
 & KNN  & $1.336\times 10^{7}$ & $130.67$ \\
 & K-means   & $1.335\times 10^{7}$ & $131.48$ \\
\bottomrule
\end{tabularx}
\end{table}

The increase in MILP solve time under generator augmentation is reflected in the underlying branch-and-bound and LP solution effort. Along the 118-bus augmentation sequence, both the number of explored branch-and-bound nodes and simplex iterations increase by approximately two orders of magnitude. Correspondingly, the baseline wall-clock time grows from $\mathcal{O}(10)\,\mathrm{s}$ to $\mathcal{O}(10^{3})\,\mathrm{s}$. 
In contrast, the proposed LLM-Assisted commitment restriction exhibits more moderate growth in solve time as system size increases, resulting in larger relative speedups on the augmented cases.

Table~\ref{tab:uc_avg_cost_time} further illustrates the scaling behavior across system levels. On the 57-bus system, the LLM-Assisted method reduces average solve time from $5.85\,\mathrm{s}$ to $3.44\,\mathrm{s}$, corresponding to a $1.70\times$ speedup. KNN and K-means reduce solve time to $3.91\,\mathrm{s}$ and $3.84\,\mathrm{s}$, respectively. Objective deviations remain small, at $0.27\%$ for the LLM-Assisted method and below $0.36\%$ for KNN and K-means.

On the RTS system, the baseline MILP requires $1149.33\,\mathrm{s}$ on average, whereas the LLM-Assisted method reduces this to $171.59\,\mathrm{s}$, yielding a $6.70\times$ speedup. KNN and K-means require $558.34\,\mathrm{s}$ and $714.54\,\mathrm{s}$, respectively, with all methods remaining within a $0.34\%$ objective deviation. 
A similar pattern is observed for the 118-bus family. The LLM-Assisted method reduces solve time from $9.06\,\mathrm{s}$ to $4.97\,\mathrm{s}$ on 118-bus and from $45.99\,\mathrm{s}$ to $23.49\,\mathrm{s}$ on 118-bus+20. On 118-bus+30, solve time decreases from $823.25\,\mathrm{s}$ to $112.49\,\mathrm{s}$, with an objective deviation of approximately $0.15\%$ relative to the MILP baseline.

\subsection{Comparison under Load Uncertainty and Long-Horizon Operations}

The uncertainty experiments are conducted on the IEEE 73-bus RTS system using a representative daily load profile that includes a significant evening peak, requiring substantial inter-temporal coordination in ramping and commitment decisions. 
To evaluate robustness under load forecast uncertainty, Gaussian perturbations are applied to the baseline hourly demand profile. Four noise levels are considered, corresponding to relative standard deviations of $10\%$, $20\%$, $30\%$, and $40\%$. For each noise level, the UC problem is solved using four approaches: the full MILP formulation, K-means-based masking, KNN-based masking, and the proposed LLM-Assisted commitment restriction.

Table~\ref{tab:acc_compare_73bus} summarizes the performance of the compared methods under the $10\%$ load perturbation scenario. All approaches produce feasible commitment schedules that satisfy power balance, minimum up/down time constraints, and ramping limits. 
Objective deviations remain small across all strategies, indicating that restricting the commitment space does not materially degrade solution quality at this noise level. 

From a computational perspective, the LLM-Assisted commitment restriction provides the largest reduction in solve time. At the $10\%$ noise level, solve time decreases by $85.11\%$, and the number of explored branch-and-bound nodes decreases by $66.41\%$ relative to the full MILP formulation. The K-means and KNN strategies also reduce computational effort, though to a lesser extent. 
The reported constraint and variable reduction percentages confirm that all masking approaches reduce the effective size of the optimization problem, with the LLM-Assisted method achieving the greatest overall reductions in this setting.

\begin{table}[h!]
\centering
\caption{Comparison of different acceleration methods on the IEEE 73 bus system.}
\label{tab:acc_compare_73bus}
\begin{tabularx}{\columnwidth}{p{30mm}cccc}
\toprule
Parameters & MILP & K-means & KNN & LLM \\
\midrule
Cost relative error (\%) & $0.00$ & $0.36$ & $<0.30$ & \bm{$<0.30$} \\
Solve time (s) & $1149.33$ & $714.54$ & $558.34$ & \bm{$171.58$} \\
Time reduction (\%) & $0.00$ & $37.83$ & $51.42$ & \bm{$85.11$} \\
Constraint reduction (\%) & $0.00$ & $33.08$ & $24.46$ & \bm{$36.73$} \\
Variable reduction (\%) & $0.00$ & $6.4$ & $12.40$ & \bm{$22.52$} \\
Node reduction (\%) & $0.00$ & $22.50$ & $37.23$ & \bm{$66.41$} \\
\bottomrule
\end{tabularx}
\end{table}

Table~\ref{tab:time_noise} reports the average and maximum solve times across the four load perturbation levels. The LLM-Assisted commitment restriction reduces solve time at all noise levels. At $10\%$ noise, the average solve time decreases from $1149.328\,$s to $171.585\,$s. At $20\%$ noise, the average solve time decreases from $803.719\,$s to $19.336\,$s, and at $30\%$ noise from $1517.875\,$s to $80.205\,$s. 
The maximum solve times exhibit similar reductions, decreasing by $91.48\%$, $96.82\%$, $93.54\%$, and $92.66\%$ across the four noise levels. These results indicate that the proposed restriction strategy maintains reduced computational effort even when load perturbations increase ramping activity and commitment adjustments.

Table~\ref{tab:cost_noise} evaluates the corresponding impact on objective value. Objective deviations remain small across all noise levels. The relative error is below $0.03\%$ at $10\%$ noise and stays within $2.99\%$, $3.80\%$, and $3.37\%$ for the $20\%$, $30\%$, and $40\%$ noise scenarios. Taken together, the results in Tables~\ref{tab:time_noise} and~\ref{tab:cost_noise} indicate that the computational reductions are achieved with limited impact on optimality under load uncertainty.

\begin{table}[h]
\centering
\caption{Time-related results of MILP and LLM-Assisted under different load noise levels on the IEEE 73-bus system.}
\label{tab:time_noise}
\begin{tabular}{llcccc}
\toprule
            & Parameters        & $10\%$ & $20\%$ & $30\%$ & $40\%$ \\
\midrule
\multirow{3}{*}{\shortstack{Ave.\\Time}}
            & Original MILP (s) & $1149.33$ & $803.72$ & $1517.87$ & $1012.36$ \\
            & LLM-Assisted (s)      & $171.59$  & $19.34$  & $80.21$    & $62.85$  \\
            & Reduced Rate (\%) & $85.11$    & $97.59$   & $94.71$     & $93.79$   \\
\midrule
\multirow{3}{*}{\shortstack{Max.\\Time}}
            & Original MILP (s) & $2678.98$ & $853.44$ & $2148.87$  & $1337.36$ \\
            & LLM-Assisted (s)      & $228.27$   & $27.17$   & $138.78$    & $98.06$    \\
            & Reduced Rate (\%) & $91.48$    & $96.82$   & $93.54$     & $92.66$    \\
\bottomrule
\end{tabular}
\end{table}

\begin{table}[h]
\centering
\caption{Cost-related results of MILP and LLM-Assisted under different load noise levels on the IEEE 73 bus system.}
\label{tab:cost_noise}
\begin{tabularx}{\columnwidth}{
l
>{\centering\arraybackslash}X
>{\centering\arraybackslash}X
>{\centering\arraybackslash}X
>{\centering\arraybackslash}X
}
\toprule
Parameters & $10\%$  & $20\%$  & $30\%$  & $40\%$  \\
\midrule
Original MILP (10$^{5}$\$)   & $33.01$ & $32.47$ & $34.17$ & $34.72$ \\
LLM-Assisted (10$^{5}$\$)        & $33.01$ & $33.44$ & $35.47$ & $35.89$ \\
Relative Error (\%)          & $<0.30$   & $2.99$  & $3.80$  & $3.37$  \\
\bottomrule
\end{tabularx}
\end{table}

As shown in Fig.~\ref{fig:uc_status_compare_day2}, black indicates the off state and white indicates the on state. 
Both panels illustrate the 24-hour commitment schedule, with each row corresponding to one generating unit. 
The LLM-Assisted schedule closely matches the MILP baseline. Base-load units remain committed throughout the day, while mid-merit and peaking units adjust their commitment during peak demand periods. Minor differences appear in a small number of marginal units with greater operational flexibility. 
The two schedule matches in $98.23\%$ of the commitment variables. This high level of agreement indicates that the proposed restriction strategy preserves the principal structural features of the UC solution while reducing computational effort.

\begin{figure}[t]
  \centering
  \begin{minipage}[t]{0.48\linewidth}
    \centering
    \includegraphics[width=\linewidth]{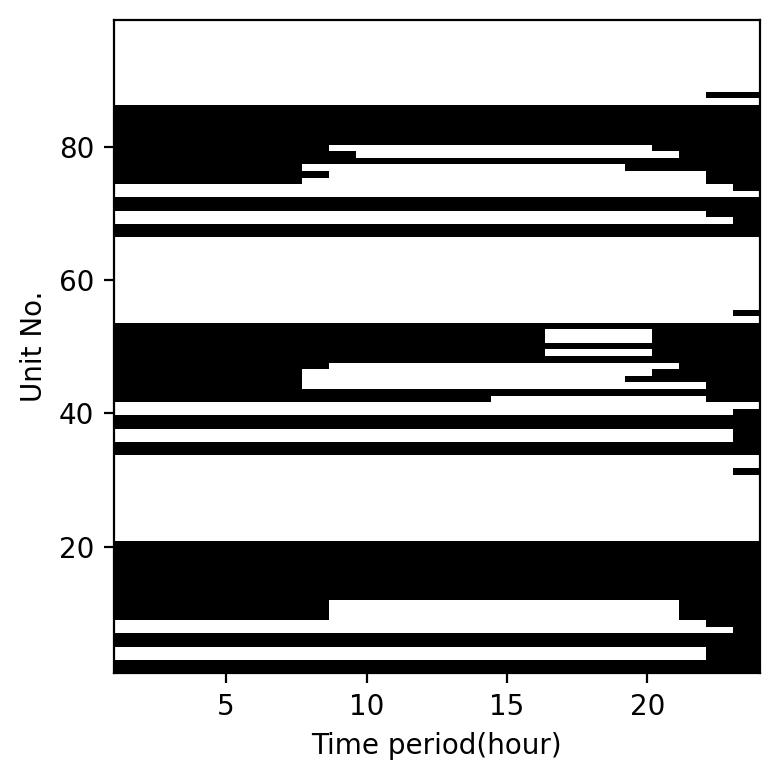}
    \small MILP
  \end{minipage}
  \hfill
  \begin{minipage}[t]{0.48\linewidth}
    \centering
    \includegraphics[width=\linewidth]{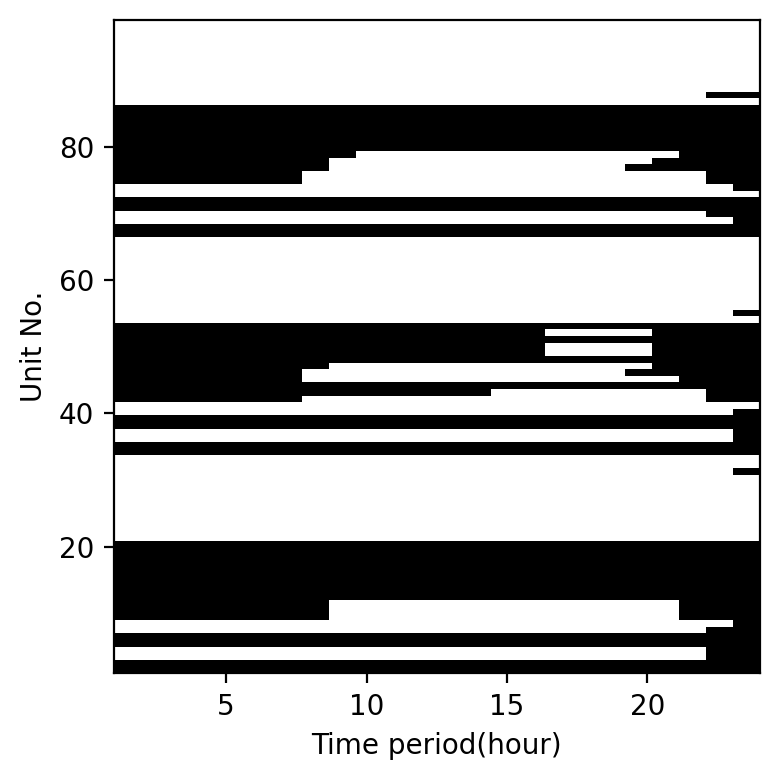}
    \small LLM-Assisted
  \end{minipage}
\caption{Comparison of 24-hour unit commitment schedules from the MILP baseline and the LLM-Assisted method. Rows correspond to generating units and columns to hours; white and black denote on and off states, respectively.}
  \label{fig:uc_status_compare_day2}
\end{figure}

Long-horizon operation requires that any acceleration strategy remain effective across sequences of daily UC problems with varying load patterns and commitment cycles. To evaluate this behavior, we apply the proposed restriction method to $365$ consecutive daily UC instances. 
The baseline MILP solver exhibits substantial day-to-day variability in solve time. Many instances exceed $1500$\,s and several exceed $2500$\,s, indicating sensitivity of the branch-and-bound process to changing load and system conditions. In comparison, the LLM-Assisted restriction method maintains solve times below $300$\,s for all $365$ instances. 
As shown in Fig.~\ref{fig:365days_time_compare}, the year-long results indicate that the proposed method reduces computational variability while preserving feasibility and maintaining objective values close to the MILP baseline. 

\begin{figure}[t]
  \centering
  \includegraphics[width=0.48\textwidth]{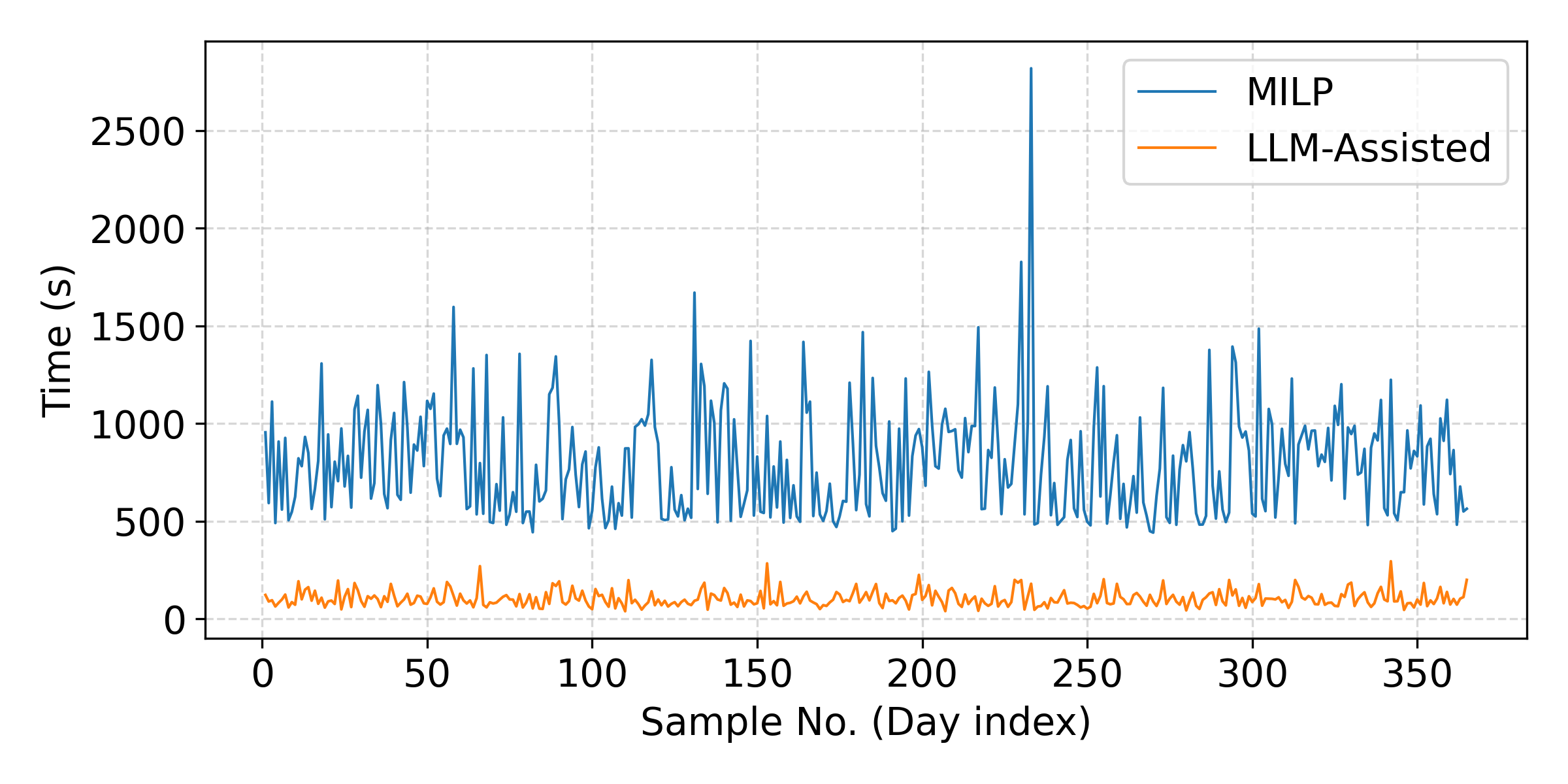}
  \vspace{-2mm}
  \caption{Solve-time comparison over 365 consecutive daily unit commitment instances. Each point corresponds to one 24-hour UC problem solved by the MILP baseline and the LLM-Assisted restriction method.}
  \label{fig:365days_time_compare}
  \vspace{-2mm}
\end{figure}

\vspace{-2mm}

\section{Conclusion} \label{sec:conclusion}

This paper presented a commitment-restriction framework for accelerating network-constrained unit commitment problems. By fixing a carefully selected subset of unit-hour commitment variables prior to solving, the proposed approach reduces the effective discrete search space without modifying the constraints. The MILP solver therefore retains full responsibility for feasibility enforcement and optimal dispatch of unfixed variables. 
Numerical studies across multiple benchmark systems demonstrated that computational gains depend on selecting structurally consistent commitment decisions. Compared with existing approaches, the proposed strategy reduces branch-and-bound effort while maintaining objective values close to the MILP baseline. The largest speedups occur in larger and more combinatorial systems, and robustness tests under load perturbations and year-long simulations indicate stable performance over extended operating conditions. 

\bibliographystyle{IEEEtran}
\bibliography{references}

\vfill

\end{document}